\documentclass{article}

\usepackage[small]{caption}
\usepackage{graphicx}
\usepackage{amsmath}
\usepackage{booktabs}
\usepackage{algorithm}
\usepackage{algorithmic}
\usepackage{amsfonts}
\usepackage{amssymb}
\usepackage{mdframed}
\usepackage{amsthm}

\newcommand{\atm}{\mathit{Atm}}

\newcommand{\dec}{\mathit{Dec}}
\newcommand{\val}{\mathit{Val}}

\newcommand{\finsubseteq}{\subseteq^{\mathrm{fin}} \hspace{-0.05cm}}

\newcommand{\tagLabel}[2]{\tag{\textbf{#1}}\label{#2}}

\renewcommand{\phi}{\varphi}

\newcommand{\allins}{[\emptyset]}
\newcommand{\someins}{\langle \emptyset \rangle}

\newcommand{\bcl}{\textsf{BCL}}

\newcommand{\pimp}{\text{$\mathtt{PImp}$}}
\newcommand{\imp}{\text{$\mathtt{Imp}$}}

\newcommand{\vx}{\mathsf{t}(x)}
\newcommand{\vy}{\mathsf{t}(y)}
\newcommand{\takevalue}[1]{\mathsf{t}({#1})}

\newcommand{\axp}{\textsf{$\mathtt{AXp}$}}
\newcommand{\waxp}{\textsf{$\mathtt{wAXp}$ }}
\newcommand{\cxp}{\textsf{$\mathtt{CXp}$}}
\newcommand{\wcxp}{\textsf{$\mathtt{wCXp}$ }}

\newcommand{\putaway}[1]{}

\newcommand{\conj}[2]{\mathsf{cn}_{#1{,}#2}}

\newcommand{\plt}{\mathit{Plt}}
\newcommand{\dfd}{\mathit{Dfd}}
\newcommand{\cb}{\mathit{CB}}

\newtheorem{definition}{Definition}
\newtheorem{proposition}{Proposition}
\newtheorem{theorem}{Theorem}
\newtheorem{corollary}{Corollary}

\newtheorem{fact}{Fact}
\newtheorem{example}{Example}

\begin{document}

\title{Bridging Case-based Reasoning in Law and Reasoning about Classifiers }



\author{Xinghan Liu$^{1}$, 
Emiliano Lorini$^{1}$, Antonino Rotolo$^{2}$ and Giovanni Sartor$^{2}$\\
\small $^{1}$IRIT-CNRS, University of Toulouse, France \\
\small $^{2}$CIRSFID, University of Bologna, Italy
}
\date{}

\maketitle

\begin{abstract}
This paper brings together two lines of research: factor-based models of case-based reasoning (CBR) and the logical specification of classifiers. 
Logical approaches to classifiers capture the connection between features and outcomes in classifier systems.
Factor-based  reasoning 
is a popular approach to reasoning by precedent in AI \& law. Horty (2011) has developed the factor-based models of precedent into a theory of precedential constraint.
In this paper we combine
the modal logic approach (binary-input classifier logic, BCL) to classifiers 
and their explanations given 
by Liu \& Lorini (2021)
with Horty's account of factor-based CBR, since
both a classifier and CBR map sets of features
to decisions or classifications.
We reformulate case bases of Horty in the language of BCL, and give several representation results. 
Furthermore, we show how notions of CBR, e.g. reason, preference between reasons, can be analyzed by notions of classifier explanation.
\end{abstract}


\section{Introduction}
This paper brings together two lines of research: the logical specification of classifiers and factor-based models of case-based reasoning (CBR). 

Logical approaches to classifiers capture 
the connection between features and outcomes in classifier systems.
They are well-suited for modeling and computing 
a large variety of explanations of a classifier's decision
\cite{shih2018formal,DBLP:conf/ecai/DarwicheH20,ignatiev2019abduction,ignatiev2020contrastive,audemard2021computational,LiuLorini2021BCL}, e.g., prime implicants, abductive, contrastive and counterfactual explanations. 
Consequently, they enable detecting biases
and discrimination in the classification process. They can thus contribute to provide controllability and explainability over automated decisions (as required, e.g., by art. 22 GDPR and by Art. 6 ECHR relative to judicial decisions).


Factor-based reasoning as provided by HYPO and CATO \cite{Ashley1990ML,Aleven2003UB} has been a most popular approach to precedential reasoning within AI \& law research. The key idea is that a case can be represented as a set of factors, where a factor is a legally relevant aspect. Factors are assumed to have a direction, i.e., to favor certain outcomes. Usually both factors and outcomes are assumed to be binary, so that each factor can be labelled with the outcome it favors (usually denoted as $\pi$,  the outcome requested by the plaintiff, and $\delta$, the outcome requested by the defendant). The party which is interested in a certain outcome in a new case can support her request by citing a past case that has the same outcome, and shares with the new case some factors supporting that outcome. The party that is interested in countering that outcome can respond with a distinction, i.e., can argue that some factors which supported that outcome in the precedent are missing in the new case or that some additional factors against that outcome are present in the new case.
In a series of contributions, John Horty \cite{Horty2004RM,HortyBench-Capon2012FB} has developed the factor-based models of precedent into a theory of precedential constraints, i.e. of how a new case must be decided, in order to preserve consistency in the case law.
In \cite{Horty2011RR,Horty2017Re} he takes into account the fact that judges may also provide explicit reasons for their choice of a certain outcome. This leads to the distinction between the result and the reason model of precedents. In the first model, the message conveyed by the case is only that  all factors supporting the case-outcome (pro-factors) outweigh all factors against that outcome (con-factors). In the second, the message is rather that the  factors for the case outcome is indicated by the judge (that may be a strict subset of the set of pro-factors) outweigh all factors against that outcome.

In this paper we shall  combine
the modal logic approach to classifiers 
and their explanations given 
in 
\cite{LiuLorini2021BCL}
with Horty's account of factor-based CBR. 
The combination is based on the fact that 
both a classifier and CBR map sets of features
to decisions or classifications. 
In this way, our contribution is at least twofold.

First, we explore the formal relation between two apparently unrelated reasoning systems. While the connection between CBR and reasoning about classifier systems is of interest in itself, we believe that, through this relation, new research perspectives can be offered, since we could in the future investigate CBR by exploiting several techniques and results from modal logic. We will see that the challenge of this paper is to adapt the formal representation of a classifier to the bidirectionality of factors in the HYPO model. Once this is solved, we can provide a logical model and a formal semantics for factor-based CBR.

Second, we contribute to investigating the idea of normative explanation. Indeed, while the philosophical literature on the concept of explanation is immense, the AI community is now paying more and more attention to it due to the development of explainable AI (XAI) \cite{XAI}. The AI\&Law community has a long tradition in this direction \cite{ATKINSON2020103387}, since transparency and justification of legal decision-making requires formalizing normative explanations. Our paper, by conncecting CBR and reasoning about classifier systems, allows for exploring different notions of explanation in law, such as abductive and contrastive explanations for the outcome suggested by the case-based reasoner. Our model could be used to build explainable case-based reasoners, which could also be deployed to reproduce and analyze the functioning of opaque predicators of the outcome of cases. We import notions such as prime implicant and contrastive explanation in the domain of XAI for classifiers to showcase how to analyze CBR in the field of XAI.
 
 \paragraph{Paper outline} 
 Section 2 presents Horty's two models of case-based reasoning (CBR). Section 3 introduces the notions of classifier model (CM) for the binary-input classifier logic ($\bcl$). Section 4 studies the connection between CBR and classifier models: we show a case base is consitent if and only if its translation into the logic $\bcl$ is satisfiable in a certain class of classifier models. In Section 5 we will instantiate that notions for classifier explanation in XAI help study case base. Finally, Section 6 discusses related work and concludes. Proofs and the axiomatics of $\bcl$ are in the appendix.

 \section{Horty's Two Models of Case-Based Reasoning}
In this section we account for the two models of case-based reasoning / precedential constraint proposed by Horty (in our language of symbols).
We simply say \emph{result model} for ``the factor-based result model of precedential constraint'' and  \emph{reason model} for ``the factor-based reason model of precedential constraint''.

\subsection{Basic notions and notations}

Let $\atm_0 = \plt \cup \dfd$, where $\plt$ and $\dfd$ are disjoint sets of factors favoring the plaintiff and defendant respectively.
In addition, let $\val = \{1, 0, ?\}$ where elements stand for \emph{plaintiff wins}, \emph{defendant wins} and \emph{indeterminacy} respectively.
Let $\dec = \{\takevalue{x}: x \in \val\}$ and read $\takevalue{x}$ as ``the actual decision/outcome (of the judge/classifier) takes value $x$''. An outcome  $\takevalue{1}$ or $\takevalue{0}$ means that, the judge is predicted to decide for the plaintiff or for the defendant (the classifies ``forces'' one of the two outcomes). The outcome  $\takevalue{?}$ means either outcome would be consistent:  the judge may develop the law in one direction or the other.
We use $\atm$ to denote $\atm_0 \cup \dec$.

We call $s \subseteq \atm_0$ a \emph{fact situation}.
A set of atoms $X$ is called a \emph{reason} for an outcome (decision) $x$ if it a set of factors all favoring the same outcome:  $X \subseteq \plt$ is a reason for $1$ and  $X \subseteq \dfd$ is a reason for $0$. A (defeasible) \emph{rule} consist of a reason and the corresponding outcome: 
$X \mapsto x$ is rule, if $X \subseteq\plt$ and $x = 1$, or $X \subseteq \dfd$ and $x = 0$.
For readability, we make a convention that, for $x \in \{0, 1\}$, let $\overline{x} = 1 - x$ and $\overline{\overline{x}} = x$.
Moreover, let $\atm_0^x = \plt$ if $x = 1$, and $\atm_0^x = \dfd$ if $x = 0$.

In the reason model, a \emph{precedent case} (precedent) is a triple $c = (s, X, x)$, where $s \subseteq \atm_0$, $X\subseteq \atm_0^x, x \in \{0, 1\}$.
In plain words, $s \cap \atm_0^x$ contains all \emph{pro-factors} in $s$ for $x$, while $s \cap \atm_0^{\overline{x}}$ all \emph{con-factors} in $s$ for $x$.
$X$ is the \emph{reason of the case}, namely a subset of the pro-factors which the judge consider sufficient to support that outcome, relative to all con-factors in the case. 

A \emph{case base} $\cb$ (for reason model) is a set of precedential cases.
When the reason contains all pro-factors within the situation   (i.e., when $c = (s, s \cap \atm_0^x, x)$) all such factors are considered equally decisive.  If a case base only contains cases of this type,  we obtain what Horty calls ``the result model'', and note such a case base $\cb^{res}$. \footnote{So we view result model as a special kind of reason model, as \cite[p. 25]{Horty2011RR} also mentioned. } 
The class of all $\cb s$ and $\cb^{\mathrm{res}}s$ are noted $\mathbf{CB}$ and $\mathbf{CB}^{\mathrm{res}}$ respectively.



\begin{example}[Running example]
Throughout the paper we will refer to the following running example taken from \cite{Prakken2021FA}. 
Let us assume  the following six factors, each of which either  favors the outcome `misuse of trade secrets' (`the
plaintiff wins') or rather favors the outcome ‘no misuse of trade secrets’ (`the defendant wins'):
the defendant had obtained the secret by deceiving the plaintiff ($\pi_1$) or by bribing
an employee of the plaintiff ($\pi_2$), the plaintiff had taken security measures to
keep the secret ($\pi_3$) , the information is obtainable elsewhere ($\delta_1$), the product is
reverse-engineerable ($\delta_2$) and the plaintiff had voluntarily disclosed the secret to
outsiders ($\delta_3$). 
Hence in our running example $\atm =\{\pi_1,\pi_2,\pi_3,\delta_1, \delta_2, \delta_3, \takevalue{0}, \takevalue{1}, \takevalue{?}\}$
%
%
Let us consider a case base $\cb^{ex} =\{c_1, c_2)$ where $c_1=(\{\pi_1, \pi_3,\delta_1,\delta_3\}, \{\pi_1\}, 1);$ $c_2=(\{\pi_2,\delta_1,\delta_3\}, \{\pi_1\},0)$, which means:
\begin{itemize}
\item $c_1$ has factors (fact situation) $s_1 = \{\pi_1, \pi_3,\delta_1,\delta_3\}$, reason $\{\pi_1\}$ and outcome $1$;
\item $c_2$ has outcome $\delta$, factors $s_2 = \{\pi_2,\delta_1,\delta_3\}$, reason $\{\delta_3\}$ and outcome $0$

\end{itemize}
\end{example}





\subsection{Consistent case base and precedential constraint for update}
A case base can be inconsistent, when two precedents
map the same fact situation to different outcomes.
Another scenario is that a consistent case base becomes inconsistent after \emph{update}, namely after expanding it with some new case.
Hence maintaining consistency is the crucial concern of case-based reasoning.
But first of all, one need define these notions.
The following definitions, except symbolic difference, are based on \cite{Horty2011RR,Prakken2021FA}.

\begin{definition}[Preference relation derived from a case]\label{def:pref case}
    Let $c = (s, X, x)$ be a case. Then the preference relation $<_c$   derived from $c$ is s.t. for any two reasons $Y$, $Y'$ favoring $x$ and $\overline{x}$ respectively, $Y' <_c Y$ if and only if $Y' \subseteq s \cap \atm_0^{\overline{x}}$ and $X \subseteq Y$.
\end{definition}

\begin{definition}[Preference relation derived from a case base]\label{def:pref case base}
    Let $\cb$ be a case base. Then  the preference relation $<_{\cb}$  derived from $\cb$ is s.t. for any two reasons $Y, Y'$ favoring $x$ and $\overline{x}$ respectively, $Y' <_{\cb} Y$ if and only if $\exists c \in \cb$ s.t. $Y' <_c Y$.
\end{definition}

\begin{definition}[(In)consistency]\label{def:Inconsist case base}
    A case base $\cb$ is inconsistent, if there are two reasons $Y, Y'$ s.t. $Y' <_{\cb} Y$ and $Y <_{\cb} Y'$. $\cb$ is consistent if it is not inconsistent.
\end{definition}

\begin{definition}[Precedential constraint]
    Let $\cb$ be a consistent case base, $X$ is a reason for $x$ in $\cb$ and applicable in a new fact situation $s'$, i.e. $X \subseteq s'$. Then updating $\cb$ with the new case $(s', X, x)$ meets the precedential constraint, if and only if $\cb \cup \{(s', X, x)\}$ is still consistent.
\end{definition}

There is more than one way to satisfy the precedential constraint, depending on how the precedents in $CB$ interacts with the new case.
The requirement of consistency dictates the outcome when the \emph{a fortiori} constraint applies: if reason $X$ for $x$ outweighs (i.e., is stronger than) reason $s \cap \atm_0^{\overline{x}}$ for $\overline{x}$, a fortiori any superset of $X$ outweighs any subset of $s \cap \atm_0^{\overline{x}}$, so that only by deciding $x$ rather than $\overline{x}$ is consistency maintained.\footnote{We generalize a fortiori constraint from only working for result model in \cite{Horty2011RR} to also for reason model. }

\begin{example}[Running example]
Let us consider two fact situations according to case base $CB^{ex}$ running example.
%

\begin{itemize}
\item \label{example:begin}
     In $s_3 = \{\pi_1, \pi_3, \delta_1\}$, only a decision for $1$ in $s_3$ is consistent with  $CB^{ex}$, since a decision for $0$ would entail that  $\{\delta_1\} >_{\cb^{ex}} \{\pi_1\}$, contrary to the preference $\{\pi_1\} >_{\cb^{ex}} \{\delta_1\}$, which is derivable from $c_1$. 
\item \label{example:begin1}
     In $s_4 = \{\pi_2, \delta_2\}$ both  $(s_4, \{\pi_2\},1)$ and $(s_4, \{\delta_2\},0)$ are consistent with $CB^{ex}$, since neither $\{\pi_2\}>_{\cb^{ex}}\{\delta_2\}$ nor $\{\delta_2\}>_{\cb^{ex}}\{\pi_2\}$.
\end{itemize}
\end{example}
\section{Classifier model of binary-input classifier logic}
In this section we introduce the language and semantics of binary-input classifier logic $\bcl$ first appeared in \cite{LiuLorini2021BCL}.
Recall that $\atm = \atm_0 \cup \dec$, where $\atm_0 = \dfd \cup \plt$, and $\dec = \{\takevalue{x}: x \in \val = \{0, 1, ?\} \}$.
The modal language  $\mathcal{L} (\mathit{Atm})$ of $\bcl$ is defined 
as:
\begin{center}\begin{tabular}{lcl}
 $\varphi$  & $::=$ & $ p \mid \takevalue{x} \mid
  \neg\varphi \mid \varphi \wedge\varphi \mid [X]\varphi,$
\end{tabular}\end{center}
where $p$ ranges over $\atm_0$, $\takevalue{x}$ ranges over $\dec$,
and $X$ is a finite subset of $\atm_0$.\footnote{$\atm$ is finite since the factors in case-based reasoning are supposed to be finite. Notice $p$ ranging over $\dfd \cup \plt$, i.e. $p$ can be some $\delta$ or some $\pi$.
$X$ can denote a reason (an exclusive set of plaintiff/defendant factors), or any subset of $\atm_0$, which is clear from the context.
Last but not least, $p$ and $\takevalue{x}$ have different statuses regarding negation: $\neg p$ means that the input variable $p$ takes value $0$, but $\neg \takevalue{x}$ merely means the output does not take value $x$: we do not know which value it takes, since the output is trinary.
}
Operator
$\langle X \rangle$
is the dual
of 
$[ X ]$
and is defined as
usual:
$\langle X \rangle \varphi =_{\mathit{def}} \neg [ X ] \neg \varphi$.
Finally, for any $X \subseteq Y \subseteq \atm_0$, the following definition syntactically expresses a valuation on $Y$ s.t. all variables in $X$ are assigned as true, while all the rest in $Y$ are false.
\begin{align*}
    \conj{X}{Y} =_{\textit{def}} \bigwedge_{p \in X } p \wedge \bigwedge_{p \in Y \setminus X} \neg p.
\end{align*}

The language 
 $\mathcal{L} (\mathit{Atm})$
is interpreted relative 
to classifier models 
defined as follows. 
\begin{definition}[Classifier model]\label{Def:ModelAltern}
	A classifier  model (CM) is a pair $C= (S, f )$ where:
	\begin{itemize}
	\item  $S \subseteq 2^{\atm_0}$
		is a set of states (or fact situations), and
	\item $f: S \longrightarrow \mathit{Val}$ is a decision (or classification) function.
	\end{itemize}
The class of classifier
models is noted
$\mathbf{CM}$.
\end{definition}

A pointed classifier model is a pair $(C,s)$ with $C= (S, f)$ a classifier model and $s \in S$.
Formulas in $\mathcal{L} (\mathit{Atm})$ are interpreted relative to a pointed  classifier model, as follows.

\begin{definition}[Satisfaction relation]\label{truthcondCM}
	Let $(C,s)$
	be a pointed  classifier model with
	$C= (S,f)$ and $s \in S$. Then:
	\begin{eqnarray*}
		(C, s) \models p & \Longleftrightarrow & p \in s , \\
		(C, s) \models \mathsf{t} (x)
	& \Longleftrightarrow & f(s)=x,\\
				(C, s) \models \neg \varphi & \Longleftrightarrow &
				(C, s) \not \models \varphi ,\\
								(C, s) \models 
								 \varphi \wedge \psi & \Longleftrightarrow &
				(C, s) \models \varphi
				\text{ and } (C, s) \models \psi ,\\
		(C,s) \models [X] \varphi
		& \Longleftrightarrow & 
		\forall s' \in S: \text{ if }
	(	s \cap X)= (s'  \cap X) 
		\text{then } (C,s') \models \varphi.
	\end{eqnarray*}
\end{definition}
\noindent A formula $\varphi$
of
$\mathcal{L} (\mathit{Atm})$
is said to 
be satisfiable
relative
to the class 
$\mathbf{CM}$
if there exists
a pointed classifier model 
$(C,s)$
with $C \in \mathbf{CM} $
such that $(C,s) \models \varphi$.
It is said to be valid if $\neg \phi$ is not satisfiable relative to $\mathbf{CM}$ and noted as $\models_{\mathbf{CM}} \phi$.

We can think of a pointed model $(C, s)$ as a pair $(s, x)$ in $f$ with $f(s) = x$.
The formula $[X] \varphi$
is true
at a state $s$
if $\varphi$
is true at all states that are 
modulo-$X$ equivalent to state $s$.
It has the \emph{selectis paribus}
(SP)
(selected
things being equal)
 interpretation ``features in $X$ being equal, necessarily $\phi$ holds (under possible perturbation on
 the other features)''. 
 $[\atm_0 \setminus X]\phi$ has the standard \emph{ceteris paribus} (CP) interpretation 
``features other than $X$ being equal, necessarily $\phi$ holds (under possible perturbation of
 the  features in $X$)''.
  Notice when $X = \emptyset$, $\allins$ is
the S5 universal modality since every state is modulo-$\emptyset$ equivalent to all states, viz. $(C, s) \models \allins \phi \iff \forall s' \in S, (C, s') \models \phi$.

\section{Representation between Consistent Case Base and CM}
In this section we shall show that the language of case base can be translated into the language $\mathcal{L}(\atm)$, hence case bases can be studied by classifier models.
More precisely, a case base is consistent 
if and only if its  translation, together with the following two formulas that we abbreviate as $\mathtt{Compl}$ and $\mathtt{2Mon}$,
is satisfiable in the class $\mathbf{CM}$:
\begin{align*}
    \mathtt{Compl} =_{\textit{def}} & \bigwedge_{X \subseteq \atm_0} \someins \conj{X}{\atm_0}
\end{align*}
\vspace{-0.5cm}
\begin{align*}
    \mathtt{2Mon} =_{\textit{def}}  \bigwedge_{x \in \{0, 1\}, X \subseteq Atm_0^x, Y \subseteq Atm_0^{\overline{x}}} & \Big( \someins ( \conj{X\cup Y}{\atm_0} \wedge \takevalue{x} ) \to \\
        & \bigwedge_{\atm_0^x \supseteq X' \supseteq X, Y' \subseteq Y} \allins ( \conj{X' \cup Y'}{\atm_0} \to \takevalue{x} ) 
        \Big)
\end{align*}


According to $\mathtt{Compl}$, \emph{every possible} situation description must be satisfied by the classifier, where a situation description is a conjunction of factors (those being present $X$) and negations of factors (those being absent, $Atm_0 \setminus X$).

$\mathtt{2Mon}$ introduces a \emph{two-way monotonicity}, which is meant to implement the \emph{a fortiori} constraint: if the classifier associates a situation $s$  to an outcome $x$, then it  must assign the same outcome to every  situation $s'$ such  that both (a)  $s'$ includes all factors for $x$ that are in $s$  and (b) $s'$ does \emph{not include} factors for ${\overline x}$ that are \emph{outside of} $s$.
This formula is meant to maintain consistency with respect to the preference relation, as Definition \ref{def:pref case} indicates:
if a situation including  factors $X$ for $x$ and factors $Y$ for $\overline x$, has outcome $x$, it means that $X>Y$. Thus it cannot be that outcome ${\overline x}$ is assigned to a  situation $s'$ including both a superset $X'\supseteq X$ of factors for $x$ and a subset $Y'\subseteq Y$ of factors for ${\overline x}$. In fact, if $X>Y$, then is must be the case that also $X'>Y'$, while a decision for ${\overline x}$ entails that    $X'<Y'$.



Let $\mathbf{CM}^{prec} = \{C = (S, f) \in \mathbf{CM} : \forall s \in S, (C, s) \models \mathtt{Compl} \wedge \mathtt{2Mon} \}$, where $\mathbf{CM}^{prec}$ means the class of CMs for precedent theory. Satisfiability and validity relative to $\mathbf{CM}^{prec}$ are defined in an analogous way as $\mathbf{CM}$.

\subsection{Representation of case base for result model}
To translate a result-model case-base  $\cb^{\mathrm{res}}$ into a classifier model $(C,f)$, we need to ensure that all precedents in the case-base are satisfied by the classifier,  with regard to both their factors and their outcome. 
\begin{definition}[Translation of case base for result model]
The translation function $tr_1$
maps each case from a case base
$\cb^{\mathrm{res}}$
to a corresponding formula
in the language $\mathcal{L}(\atm)$.
It is defined as follows:
    \begin{align*}
        tr_1(s, s \cap \atm_0^x, x) =_{\textit{def}} \someins (\conj{s}{\atm_0} \wedge \takevalue{x}).
    \end{align*}
    We generalize it 
    to the entire case base $\cb^{\mathrm{res}}$
 as follows:
    \begin{align*}
        tr_1(\cb^{\mathrm{res}}) =_{\textit{def}} \bigwedge_{(s, s \cap \atm_0^x, x) \in \cb} tr(s, s \cap \atm_0^x, x).
    \end{align*}
\end{definition}

Therefore, in the result model a precedent $(s, s\cap \atm_0^x, x)$ is viewed as a situation $s$ being  classified by $f$ as $x$. 
\begin{example}[Running example] The case $(\{\pi_1,\pi_2,\delta_1\},\{\pi_1,\pi_2\}, 1\})$ would be translated as 
$\someins (\pi_1\land \pi_2 \land \delta_1 \land \neg \pi_3 \land \neg \delta_2 \land \neg \delta_3\land \takevalue{1})$, which means that $f(\pi_1, \pi_2, \delta_1) = 1$
\end{example}


\subsection{Representation for the reason model}
In  translations  for the reason model we need to capture the role of reasons. This is obtained by ensuring that for every precedent $(s, X, x)$, not the fact situation $s$ directly, but the one consisting only of reason $X$ and all $\overline{x}$-factors in $s$ (i.e. $s \cap \atm_0^{\overline{x}}$) is classified as $x$. It reflects that the precedent finds $x$-factors outside of $X$ dispensable for the outcome.

\begin{definition}[Translation of case base for reason model]
 The translation function $tr_2$
maps each case from a case base
$\cb$
to a corresponding formula
in the language $\mathcal{L}(\atm)$.
It is defined as follows:
    \begin{align*}
        tr_2(s, X, x) =_{\textit{def}} \someins (\conj{X \cup (s \cap \atm_0^{\overline{x}})}{\atm_0} \wedge \takevalue{x}).
    \end{align*}
     We generalize it 
    to the entire case base $\cb$
 as follows:
    \begin{align*}
        tr_2(\cb) =_{\textit{def}} \bigwedge_{(s, X, x) \in \cb} tr_2(s, X, x).
    \end{align*}
\end{definition}

Note that the function $tr_1$ for the result model is a  special case of the function $tr_2$ for the reason model, since $((s \cap \atm_0^x) \cup (s\cap \atm_0^{\overline{x}}) = s$

\begin{fact}
    $tr_1(s, s \cap \atm_0^x, x) = tr_2(s, s \cap \atm_0^x, x)$. 
\end{fact}



The formulas $\mathtt{2Mon}$ and  $\mathtt{Compl}$ require that the the outcome $x$ supported by reason $X$ in a precedent is assigned to all possible cases including $X$ that do not contain additional factors against $x$. 
If both formulae are satisfiable then the case base is consistent, as stated by the following theorem.

\begin{theorem}\label{theor: CB and CM tr'}
    Let $\cb \in \mathbf{CB}$ be a case base. 
    Then, $\cb$ is consistent if and only if $tr_2(\cb) $
    is satisfiable
    in $\mathbf{CM}^{prec}$.
\end{theorem}

In light of the theorem and the fact above, the representation of case base for result model turns to be a corollary.

\begin{corollary}
    Let $\cb^{\mathrm{res}} \in \mathbf{CB}^{\mathrm{res}}$ be a case base for the result model. 
    Then, $\cb^{\mathrm{res}}$ is consistent if and only if $tr_1(\cb^{\mathrm{res}}) $ is satisfiable in $\mathbf{CM}^{prec}$.
\end{corollary}

Similarly, the precedential constraint can also be represented as a corollary.

\begin{corollary}
    Let $\cb \in \mathbf{CB}$ be a consistent case base and $(s', X, x)$ a case. Updating $\cb$ with  $(s', X, x)$ meets the precedential constraint, if and only if $tr_2(\cb) \wedge tr_2(s', X, x)$ is satisfiable in $\mathbf{CM}^{prec}$.
\end{corollary}
\begin{example}[Running example] Case $c_3 = (\{\pi_1, \pi_2, \delta_2\}, \{\delta_2\}, 0)$ is incompatible with the $CB^{ex}$. According to $tr_2(\cb^{ex}\cup\{c_3\})$,  $\mathtt{2Mon}$ and  $\mathtt{Compl}$,  
the fact situation $\{\pi_1, \pi_2, \delta_1\}$ should be classified both as $1$, based on $CB^{ex}$, and $0$, based on $c_3$.
\end{example}

\section{Explanations}

The representation results above pave the way to providing explanations for the outcomes of cases. 
For this purpose it is necessary to introduce the following notations.
Let $\lambda$ denote a conjunction of finitely many literals, where a literal is an atom $p$ (positive literal) or its negation $\neg p$ (negative literal).
We write $\lambda \subseteq \lambda'$, call $\lambda$ a part (subset) of $\lambda'$, if all literals in $\lambda$ also occur in $\lambda'$; and $\lambda \subset \lambda'$ if $\lambda \subseteq \lambda'$ but not $\lambda' \subseteq \lambda$.
We write $Lit(\lambda), Lit^+(\lambda), Lit^-(\lambda)$ to mean all literals, all positive literals and all negative literals in $\lambda$ respectively.
By convention $\top$ is a term of zero conjuncts.
In the glossary of Boolean classifier (function), $\lambda$ is called a \emph{term} or \emph{property} (of the instance $s$).
The set of terms
is noted $ \mathit{Term}$. 
%
A key role in our analysis is played by  the notion of a  (prime) implicant, i.e., a  (subset-minimal) term which makes a classification necessarily true.
\begin{definition}[Implicant (Imp) and prime implicant (PImp)]\label{def:PI}
We write $\imp(\lambda, x)$ to mean  that $\lambda$ is an \emph{implicant} for $x$ and define it as 
$    \imp(\lambda, x) =_{\mathit{def}}
    \allins (\lambda \to \takevalue{x}).$
We write 
$\pimp(\lambda, x)$
to mean that $\lambda$ is a \emph{prime implicant} for $x$
and define it as 
\begin{align*}
\pimp(\lambda, x) =_{\mathit{def}} 
    \allins \Big( \lambda \to \big(\takevalue{x} \wedge 
    \bigwedge_{p \in \atm(\lambda)} \langle \atm(\lambda)\setminus\{p\}\rangle \neg \takevalue{x}  \big) \Big). 
\end{align*}
\end{definition}

According to the definition, $\lambda$ being an implicant for $x$ means that any state $s$ verifying $\lambda$ is necessarily classified as $x$ (necessity); and $\lambda$ being a prime implicant for $x$ means that any proper subset of $\lambda$ is not an implicant for $x$ (minimality).\footnote{Notice that we have not fully used the expressive power of $[X] \phi $ and $\langle X \rangle \phi $ until now for minimality. The intuitve meaning of $\langle \atm(\lambda)\setminus\{p\}\rangle \neg \takevalue{x}$ in the formula is that even one variable $p$ in $\lambda$ does not keep its actual value, the classification possibly no longer be $x$. } 
Implicants explain the classifier in the sense that to know an implicant satisfied at a state is to know the classification of the state.

Intuitively, for a case base containing precedent $(s, X, x)$ to be consistent, $s$ must be incompatible with every prime implicant $\lambda$ for $\overline{x}$. To guarantee that, either $\lambda$ must have some literal $\neg p$, where $p$ is in $X$ and hence is true at $s$; or $\lambda $ must have some literal $p$, where $p \notin s \cap \atm_0^{\overline{x}}$ and hence is false at $s$.

\begin{proposition}\label{prop: CB and Imp}
    Let $\cb$ be a consistent case base and $(s, X, x) \in \cb$, and $C \in \mathbf{CM}^{prec}$ s.t. $(C, s) \models tr_2(\cb) $. Then, 
    $\forall \lambda \in \mathit{Term}, $, if $(C, s) \models \pimp(\lambda, \overline{x})$, then either $X \cap \atm(Lit^-(\lambda)) \neq \emptyset$ or $s \cap \atm_0^{\overline{x}} \nsupseteq \atm(Lit^+(\lambda)) $. 
\end{proposition}

\begin{example}
    Let $C = (S, f) \in \mathbf{CM}^{prec}$ and $tr_2(\cb^{ex})$ is satisfiable in $C$. Obviously $\pi_1$ cannot be PImp for $0$, otherwise $f(s_1) = 0$, contrary to $c_1$.
    Also $\neg \delta_2 \wedge \pi_2$ cannot be PImp for $1$, otherwise $f(\{\pi_2, \delta_1, \delta_3\}) = 1$, contrary to $c_2$.
\end{example}


In XAI, people \cite{shih2018formal,DBLP:conf/ecai/DarwicheH20,ignatiev2019abduction} also focus on ``local'' (prime) implicants, namely (prime) implicants true at a given state. We adopt the definitions in \cite{ignatiev2019abduction,huang2022tractable} and express them in $\mathcal{L}(\atm)$ as follows.

\begin{definition}[Abductive explanation (AXp) and weak abductive explanation (wAXp) ]
 We write 
$\axp(\lambda, x)$
to mean that
$\lambda$ \emph{abductively explains}
the decision $x$
and define it as 
$\axp(\lambda, x) =_{\mathit{def}} 
  \lambda \wedge \pimp(\lambda, x).$
We write $\waxp(\lambda, x)$ to mean that $\lambda$ \emph{weak-abductively explains} the decision $x$ and define it as 
$  \waxp(\lambda, x) =_{\mathit{def}} \lambda \wedge \imp(\lambda, x).$
\end{definition}

The proposition below states that to be the reason (of a fact situation) is to be the positive part of some weak AXp of that situation.
Notice a reason is not always the positive part of some AXp, since reason in precedent does not in general respect minimality.
\begin{proposition}\label{prop: reason implies axp }
    Let $\cb$ be a consistent case base, $(s, X, x) \in \cb$, and $C \in \mathbf{CM}^{prec}$ s.t. $(C, s) \models tr_2(\cb) $. Then $\exists \lambda \in \mathit{Term}$ s.t. $ \atm(\mathit{Lit}^+(\lambda)) = X$ and $(C, s) \models \waxp(\lambda, x) $.
\end{proposition}

In fact, we always know a weak AXp for a precedent $(s, X, x)$, which is the conjunction of all factors in $X$ and negations of all $\overline{x}$-factors that are not in $s$.

\begin{proposition}\label{prop: reason gives at least one AXp}
    Let $\cb$ be a consistent case base, $(s, X, x) \in \cb$, and $C \in \mathbf{CM}^{prec}$ s.t. $(C, s) \models tr_2(\cb) $. Then, $(C, s) \models \waxp( \conj{X}{(X \cup \atm_0^{\overline{x}}) \setminus (s \cap \atm_0^{\overline{x}})} , x)$.
\end{proposition}

\begin{example}
    Let $C \in \mathbf{CM}^{prec}$ be a model of $tr_2(\cb^{ex})$. Then we have $(C, s_1) \models \waxp(\pi_1 \wedge \neg \delta_2, 1)$ and $(C, s_2) \models \waxp(\delta_2 \wedge \neg \pi_1 \wedge \neg \pi_2, 0)$.
    Notice that $(C, s_2) \models \neg \waxp(\delta_2, 0) $, because e.g. $(C, s_1) \models \delta_2 \wedge \neg \takevalue{0}$.
\end{example}





The idea of contrastive explanation is dual with abductive explanation, since it points to a minimal part of a situation whose change would falsify the current decision, and the duality between their weak versions is similar \cite{huang2022tractable}. A conjunction of literals $\lambda$ is a contrastive explanation  for outcome $x$ in situation $s$, if the following conditions are satisfied: (a) $\lambda$ is true at $s$, and $s$ has the outcome $x$, (b) if all literals in $\lambda$ were false then the outcome would be different, (c) $\lambda$ is the subset-minimal literals satisfiying (a) and (b).
A weak contrastive explanation is only based on  conditions (a) and (b). 

\begin{definition}[Contrastive explanation (CXp) and weak contrastive explanation (wCXp)]
 We write 
$\cxp(\lambda, x)$
to mean that
$\lambda$ \emph{constrastively explains}
the decision $x$
and define it as
\begin{align*}
\cxp(\lambda, x) =_{\mathit{def}}  &
  \lambda \wedge
        \langle \mathit{Atm}_0 \setminus \mathit{Atm}(\lambda) \rangle \neg \vx      \wedge 
        \bigwedge_{p \in Atm(\lambda)} [(\mathit{Atm}_0 \setminus Atm(\lambda)) \cup \{p\}] \vx. 
\end{align*}
We write $\wcxp(\lambda, x)$ to mean that $\lambda$ \emph{weak-contrastively explains} the decision $x$ and define it as 
$\wcxp(\lambda, x) =_{\mathit{def}} \lambda \wedge \takevalue{x} \wedge \langle \atm_0 \setminus \atm(\lambda) \rangle \neg \takevalue{x}.$
\end{definition}


Intuitively speaking, we can test whether $\lambda$ is a wCXp of situation $s$ having outcome $x$ by ``flipping'' its positive literals to negative, and negative to positive, and observe if the resulting state is classified differently from $x$. CXp is the subset-minimal wCXp.

Weak CXps can be used to study the preferences between reasons in a case base.
The next proposition inicates that
given a precedent $(s, X, x)$, if the absence of $Y$ at $s$ by itself \emph{alone} can weakly contrastively explain $x$, then $Y$ is ``no weaker than'' $X$ in $\cb$.

\begin{proposition}\label{prop: wCXp and update 2}
    Let $\cb$ be a consistent case base and $(s, X, x) \in \cb$, and $C \in \mathbf{CM}^{prec}$ s.t. $(C, s) \models tr_2(\cb)$. If $(C, s) \models \wcxp(\conj{\emptyset}{Y}, x) $, then it is not that $Y \cup (s \cap \atm_0^{\overline{x}}) <_{\cb} X$.
\end{proposition}

\begin{example}
    Let $C \in \mathbf{CM}^{prec}$ be a model of $tr_2(\cb^{ex})$. Since $\{\delta_3\} <_{\cb^{ex}} \{\pi_1\} $, we have $(C, s_2) \models \wcxp(\conj{\emptyset}{\{ \pi_1 \}}, 0)$. Indeed $f(\pi_1, \delta_1, \delta_3) = 0$ by $\mathtt{2Mon}$ according to $s_1$.
\end{example}

\section{Related work and conclusion}
In this paper, we have shown that through the concept of classifier in \cite{LiuLorini2021BCL} a novel logical model of factor-based reasoning can be provided, which allows a rigorous analysis of case bases and of the inferences they support.

As noted in the introduction, our work is based upon the case-based reasoning models of HYPO and CATO \cite{Ashley1990ML,Aleven2003UB} and upon the analysis of precedential constraint by Jeff Horty \cite{Horty2011RR,HortyBench-Capon2012FB}. Further approaches exist that make use of logic in reasoning with cases. For instance, \cite{prakken1998modelling} procided a factor-based model based on formal defeasible argumentation. More recently \cite{zheng2020case,zheng2020precedent} represent precedents as propositional formulas and compare precedents by (propositional) logical entailment.

However, they do not fully use the power of logic,
in the sense that a proof theory (axiomatics) for reasoning with precedents is not provided.
By contrast, besides the semantic framework presented here, we can 
make syntactic derivations of properties of CBR using the axiomatics of $\bcl$ (see in Appendix).

Moreover, our representation results allow for exploring different notions of explanation, such as abductive and contrastive explanations. We can accordingly explain why a case-based reasoning suggests a particular  outcome (rather then a different one) in a new case. Thus, out model could be used to build explainable case-based reasoners, which could also be deployed to reproduce and analyse the functioning of opaque predictors of the outcome of cases.
Thus, CBR is brought into a broader context of classifier systems. Thus, we connect three lines of research: legal case-based reasoning,  AI\&Law approaches on  to explanation \cite{ATKINSON2020103387},  techniques and results developed in the context of XAI.

In future work we will deepen the relation between classifiers, explanations, and reasoning with legal precedents.  Interesting developments pertain to  addressing analogical reasoning beyond the a fortiori constraint considered here and to deploying ideas of explanation to extract knowledge out of cases (e.g., to determine the direction of factors and the way in which they interact).

\bibliographystyle{plain}
\bibliography{CMandCBR}

\newpage
\appendix

\section{Proofs}
\subsection{Proof of Theorem \ref{theor: CB and CM tr'}}
\begin{proof}
    uppose $\cb$ is consistent. We construct a classifier model $C = (S, f)$ s.t. $S = 2^{\atm_0}$, and $\forall s \in S$, we have
    \begin{align*}
        f(s) = \left\{ \
        \begin{array}{lll}
           x \text{ for } x \in \{0, 1\}, & \text{ if } \exists (s', X, x) \in \cb  \text{ s.t. } s \cap \atm_0^x \supseteq X \text{ and } s \cap \atm_0^{\overline{x}} \subseteq s' \cap \atm_0^{\overline{x}}; \\
            ? & \text{ otherwise.}
        \end{array}
        \right.
    \end{align*}
    Obviously $(C, s) \models \mathtt{Compl}$ since $S = 2^{\atm_0}$. We need show that $(C, s) \models \mathtt{2Mon}$. Suppose the opposite towards a contradiction. 
    W.l.o.g., suppose $\exists s = X \cup Y \in S, f(s) = 0$, where $X \subseteq \dfd, Y \subseteq \plt$ and $\exists s' = X' \cup Y'$ s.t. $X' \supseteq X, Y' \subseteq Y$ but $f(s') = 1$.
    By the construction of $f$, since $f(s) = 0, \exists c_0 = (s_0, X_0, 0) \in \cb$ s.t. $X \supseteq X_0$ and $Y \subseteq s_0 \cap \plt$.
    Similarly, $\exists c_1 = (s_1, X_1, 1) \in \cb$ s.t. $ Y' \supseteq X_1 $ and $X' \subseteq s_1 \cap \dfd$. By transitivity of $\subseteq$, one can check that $X <_{c_1} Y$ and $Y <_{c_0} X$, which contradicts that $\cb$ is consistent.
    
    For the other direction, suppose $\cb$ is inconsistent, we show that $tr_2(\cb) \wedge \mathtt{Compl} \wedge \mathtt{2Mon} $ is unsatisfiable.
    Since $\cb$ is inconsistent, by definition we shall have $Y_0 <_{\cb} Y_1$ and $Y_1 <_{\cb} Y_0$.
    W.l.o.g., assume in $\cb$ there are two precedents $c_0 = (s_0, X_0, 0), c_1 = (s_1, X_1, 1)$ s.t. $Y_0 <_{c_1} Y_1, Y_1 <_{c_0} Y_0$.
    Unravel the definition we have $Y_0 \subseteq s_1 \cap \dfd $ and $X_1 \subseteq Y_1$; $Y_1 \subseteq s_0 \cap \plt$ and $X_0 \subseteq Y_0$.
    
    Now towards a contradiction suppose $(C, s)$ be a pointed CM s.t. $(C, s) \models tr_2(\cb) \wedge \mathtt{Compl} \wedge \mathtt{2Mon}$. Consider the state $s_2 = Y_0 \cup Y_1$. Since $(C, s) \models \mathtt{Compl}$ we always have $s_2 \in S$. Then by $\mathtt{2Mon}$, we have $f(s_2) = 0$ with respect to $s_0$, since $s_2 \cap \dfd = Y_0 \supseteq X_0 \supseteq s_0 \cap \dfd$ and $s_2 \cap \plt = Y_1 \subseteq s_0 \cap \plt$. But also by $\mathtt{2Mon}$ we have $f(s_2) = 1$ with respect to $s_1$. Hence $f$ fails to be functional, a contradiction that we want.
\end{proof}

\subsection{Proof of Proposition \ref{prop: CB and Imp}}
\begin{proof}
    Suppose towards a contradiction that $\exists \lambda, (C, s') \models \pimp(\lambda, \overline{x})$, $X \cap \atm(Lit^-(\lambda)) = \emptyset$ and $s \cap \atm_0^{\overline{x}} \supseteq \atm(Lit^+(\lambda)) $. Then $\lambda \wedge \conj{X}{X}$ is consistent. By $\mathtt{Compl}$ we have some $s^\dagger \in S, (C, s^\dagger) \models \lambda \wedge \conj{X}{X}$ and $f(s^\dagger) = \overline{x}$ since $\lambda$ is a PImp. However, by virtue of $\mathtt{2Mon}$ according to $s$ we shall have $f(s^\dagger) = x$, a contradiction that we want.
\end{proof}

\subsection{Proof of Proposition \ref{prop: reason implies axp }}
\begin{proof}
    If it were no such $\lambda$, then we would have some $s' \in S$ s.t. $X \subseteq s'$, $s' \cap \atm_0^{\overline{x}} \subseteq s \cap \atm_0^{\overline{x}}$ and $f(s') \neq x$. However, this contradicts $\mathtt{2Mon}$.
    Notice that if the classifier is trivial, i.e. $\forall s' \in S, f(s') = x$, then we have $(C, s) \models \axp(\top, x)$ and by definition of term, $X \supseteq \atm(\mathit{Lit}(\top)) = \emptyset$.
\end{proof}

\subsection{Proof of Proposition \ref{prop: reason gives at least one AXp}}
\begin{proof}
    It is easy to see that if it were not the case, then $\mathtt{2Mon}$ would not be true in $C$.
\end{proof}

\subsection{Proof of Proposition \ref{prop: wCXp and update 2}}
\begin{proof}
By the antecedent we have $f(s \cup Y') \neq x$ for some $Y' \subseteq Y$. Now suppose $Y \cup (s \cap \atm_0^{\overline{x}}) <_{\cb} X$, then we should have $f(X \cup Y \cup (s \cap \atm_0^{\overline{x}})) = x$. However, thus by $\mathtt{2Mon}$ it would be $f(s \cup Y') = x$, a contradiction.
\end{proof}

\section{Axiomatics}
We present the axiomatics of $\bcl$, which is proven sound and complete relative to $\mathbf{CM}$ in \cite{LiuLorini2021BCL}.
\begin{definition}[Axiomatics of $\bcl$]\label{axiomatics}
We define  $\bcl$
(Binary Classifier Logic)
to be the extension of classical
propositional logic given by the following
 axioms and rules of inference:
\begin{align}
& \big( [\emptyset] \varphi
\wedge [\emptyset] (\varphi \rightarrow \psi) \big)
\rightarrow [\emptyset] \psi
 \tagLabel{K$_{[\emptyset]}$}{ax:Kbox}\\
 & [\emptyset] \varphi
\rightarrow  \varphi
 \tagLabel{T$_{[\emptyset]}$}{ax:Tbox}\\
  & [\emptyset] \varphi
\rightarrow  [\emptyset][\emptyset] \varphi
 \tagLabel{4$_{[\emptyset]}$}{ax:4box}\\
   & \varphi
\rightarrow  [\emptyset] \langle \emptyset \rangle \varphi
 \tagLabel{B$_{[\emptyset]}$}{ax:Bbox}\\
   &  [X ]\varphi \leftrightarrow 
\bigwedge_{Y \subseteq X } \big(  \conj{Y }{X}
\rightarrow  [\emptyset ]( \conj{Y }{X} \rightarrow \varphi)  \big)
 \tagLabel{Red$_{[\emptyset]}$}{ax:Redbox}\\
 &\bigvee_{x \in \mathit{Val}}\vx
 \tagLabel{AtLeast}{ax:Leastx}\\
 & \vx \to \neg \vy \text{ if }x \neq y
 \tagLabel{AtMost}{ax:Mosttx}\\
& \bigwedge_{ Y \finsubseteq \atm_0}\Big(
\big(\conj{Y }{\atm_0} \wedge \vx \big) \rightarrow 
   [\emptyset]\big(\conj{Y }{\atm_0} \rightarrow
\vx \big) \Big)
 \tagLabel{Funct}{ax:functX}\\
& \frac{\phi \to \psi,  \hspace{0.25cm} \phi}{\psi}
\tagLabel{MP}{rule:MP}\\
&  \frac{\varphi }{[\emptyset] \varphi }
 \tagLabel{Nec$_{[\emptyset]}$}{rule:Necbox}
\end{align}
\end{definition}

We write $\vdash_\bcl \phi$ to mean that $\phi$ is a $\bcl$ theorem, that is, $\phi$ is derivable from axioms and rules of inferences of $\bcl$.

We show the proposition below proof-theoretically to exercise the axiomatics $\bcl$.
\begin{proposition}\label{prop: case has Imp }
    We have the following validity
    \begin{align*}
        \models_{\mathbf{CM}} tr_2(s, X, x) \to \bigvee_{\lambda \in \mathit{Term}} \big(\imp(\lambda, x) \wedge (\lambda \to \conj{X}{X}) \big).
    \end{align*}
\end{proposition}
\begin{proof}
    We prove by deriving $\vdash_\bcl tr_2(s, X, x) \to \bigvee_{\lambda \in \mathit{Term}} \big( \imp(\lambda, x) \wedge (\lambda \to \conj{X}{X} \big).$ For readability we write $s^x_X$ for $X \cup (s \cap \atm_0^{\overline{x}})$.
\begin{enumerate}
    \item $\vdash_\bcl \bigwedge_{\lambda \in \mathit{Term}} \neg (\imp(\lambda, x) \to \conj{X}{X}) \to \neg (\imp(\conj{s^x_X}{\atm_0}, x) \wedge ( \conj{s^x_X}{\atm_0} \to \conj{X}{X}) $ \\
    \text{ \small by the fact that $\conj{s^x_X}{\atm_0}$ is a term}
    \item $\vdash_\bcl \neg (\imp(\conj{s^x_X}{\atm_0}, x) \wedge ( \conj{s^x_X}{\atm_0} \to \conj{X}{X} ) \to \neg \imp(\conj{s^x_X}{\atm_0}, x)$\\
    \text{ \small by the fact that $\conj{s^x_X}{\atm_0} \to \conj{X}{X}$ }
    \item $\vdash_\bcl \neg \imp(\conj{s^x_X}{\atm_0}, x) \to \neg \allins (\conj{s^x_X}{\atm_0} \to \takevalue{x}) $\\
    \text{ \small by definition of Imp}
    \item $\vdash_\bcl tr_2(s, X, x) \to \allins (\conj{s^x_X}{\atm_0} \to \takevalue{x})$ \\
    \text{ \small by definition of $tr_2(s, X, x)$, $\mathbf{Funct}, \mathbf{MP}$ and a theorem proven below}
    \item $\vdash_\bcl (tr_2(s, X, x) \wedge \neg \imp(\conj{s^x_X}{\atm_0}, x)) \to \bot  $\\
    \text{ \small from 3, 4 by propositional logic }
    \item $\vdash_\bcl tr_2(s, X, x) \to \bigvee_{\lambda \in \mathit{Term}} \big( \imp(\lambda, x) \wedge (\lambda \to \conj{X}{X} \big)$ \\
    \text{ \small from 2, 5 by propositional logic}
\end{enumerate}
The theorem used in 4 is $\vdash_\bcl \someins (\conj{Y}{\atm_0} \wedge \takevalue{x}) \to \allins (\conj{Y}{\atm_0} \to \takevalue{x}) $
and is derived as follows.
\begin{enumerate}
    \item $\vdash_\bcl \takevalue{y} \to \neg \takevalue{x} $ for $x \neq y$
    \item $\vdash_\bcl (\neg \conj{Y}{\atm_0} \vee \takevalue{y}) \to (\neg \conj{Y}{\atm_0} \vee \neg \takevalue{x} ) $ \\
    \text{ \small by propositional logic}
    \item $\vdash_\bcl \allins (\neg \conj{Y}{\atm_0} \vee \takevalue{y}) \to \allins (\neg \conj{Y}{\atm_0} \vee \neg \takevalue{x} ) $ \\
    \text{ \small from 1, 2 by $\mathbf{Nec}, \mathbf{K}$ and $\mathbf{MP}$ }
    \item $\vdash_\bcl \allins( \conj{Y}{\atm_0} \to \takevalue{y} ) \to \allins (\conj{Y}{\atm_0} \to \neg \takevalue{x} )$
    \item $\vdash_\bcl \bigvee_{x' \in \val \setminus \{x\}} \allins (\conj{Y}{\atm_0} \to \takevalue{x'}) \to \allins(\conj{Y}{\atm_0} \to \neg \takevalue{x}) $
    \item $\vdash_\bcl \someins (\conj{Y}{\atm_0} \wedge \takevalue{x}) \to \bigwedge_{x' \in \val\setminus\{x\}} \someins (\conj{Y}{\atm_0} \wedge \neg \takevalue{x'}) $ \\
    \text{ \small by countraposition of 5}
    \item $\vdash_\bcl \bigwedge_{x' \in \val \setminus \{x\} } (\someins (\conj{Y}{\atm_0} \wedge \neg \takevalue{x}) \to \allins (\conj{Y}{\atm_0} \to \neg \takevalue{x'}) $ \\
    \text{ \small by propositional logic and countraposition of $\mathbf{Funct}$}
    \item $\vdash_\bcl \someins (\conj{Y}{\atm_0} \wedge \takevalue{x} ) \to \bigwedge_{x'\in \val \setminus \{x\}} \allins(\conj{Y}{\atm_0} \to \neg \takevalue{x'}) $ \\
    \text{ \small from 6, 7 by $\mathbf{MP}$ }
    \item $\vdash_\bcl \bigwedge_{x'\in \val \setminus \{x\}} \allins(\conj{Y}{\atm_0} \to \neg \takevalue{x'}) \leftrightarrow \allins (\conj{Y}{\atm_0} \to \takevalue{x})  $ \\
    \text{ \small by $\mathbf{AtLeast, AtMost, Nec}$, $\mathbf{K}$ and propositional logic }
    \item $\vdash_\bcl \someins (\conj{Y}{\atm_0} \wedge \takevalue{x}) \to \allins (\conj{Y}{\atm_0} \to \takevalue{x}) $ \\
    \text{ \small from 8, 9 by $\mathbf{MP}$ }
\end{enumerate}
\end{proof}

\end{document}